\documentclass[11pt]{llncs}

\usepackage{mlbook}

\usepackage{hyperref}
\usepackage{amsfonts}
\usepackage{amsmath}
\usepackage{amstext}
\usepackage{latexsym}
\usepackage{color}
\usepackage{nicefrac}       
\usepackage{microtype}      
\usepackage{llncsdoc}

\hypersetup{
  colorlinks   = true,
  urlcolor     = blue,
  linkcolor    = blue,
  citecolor    = blue
}

\ignore{
\makeatletter
\newtheorem*{rep@theorem}{\rep@title}
\newcommand{\newreptheorem}[2]{%
\newenvironment{rep#1}[1]{%
 \def\rep@title{#2 \ref{##1}}%
 \begin{rep@theorem}}%
 {\end{rep@theorem}}}
\makeatother

\newtheorem{theorem}{Theorem}
\newreptheorem{theorem}{Theorem}

\newreptheorem{lemma}{Lemma}

\newreptheorem{corollary}{Corollary}

\newreptheorem{proposition}{Proposition}
}

\title{Online Learning Algorithms for \\Statistical Arbitrage}
\author{Christopher Mohri}
\authorrunning{Christopher Mohri}
\institute{Hunter College High School, \\
New York, NY 10128,\\
\email{christopher.mohri@gmail.com}}

\begin{document}

\maketitle

\section{Introduction} 
\label{sec:arbitrage}

Arbitrage is the risk-free method of making profit from exploiting
price differences in different markets. For example, if one stock is
trading at a higher price in one market than another, one could buy
the stock for the lower price on one market and sell it for the higher
price on the other, thereby making profit without taking risks. These
pricing disparities have become increasingly hard to capitalize on
as they only appear for very short periods of time with the
advancements in technology and high-frequency trading. Only those who
can recognize and take advantage of arbitrage opportunities first can
benefit, turning it into a winner-takes-all situation. This has made
it difficult to make consistent profit from price discrepancies, as
one needs to recognize them quickly and be the first to leverage
them. Yet, arbitrage is a necessary tool in the marketplace as it
quickly eliminates market inefficiencies and keeps prices uniform
across markets
\citep{Avellaneda2011,DamghaniKos2013,Lo2010,FernholzMaguire2007,
  AvellanedaLee2008, Wu2007}.

One type of arbitrage is taking advantage of the difference of the
cost of an ETF against the summation of the prices of the stocks in
the underlying basket. If the cost of an ETF exceeds the cost of the
underlying basket of stocks, then one can buy the individual stocks
for the prices on the market, and sell the ETF in order to make a
profit.

Another type of arbitrage is through taking advantage of discrepancies
within currency conversions, such as in a currency triangle. For
example, if one converts a certain currency to another, and then again
to another, and then back to the original currency, then the resulting
balance does not necessarily equal the initial balance if there exists
a discrepancy. Taking advantage of such a discrepancy results in
arbitrage.

Another key type of arbitrage is known as \emph{statistical
  arbitrage}. Statistical arbitrage relies on historical data and
statistics to determine relatively risk-free strategies, although of
course this may not always be exactly the case. One simple example of
statistical arbitrage is pairs trading. Pairs trading consists of
identifying correlations between two or more stocks. When stocks have
historically appeared to be strongly correlated, then we can assume
that they ultimately will converge when they currently seem to
diverge. For example, if stock A and stock B are correlated (they may
even be in the same field, such as Coca Cola and Pepsi), and the price
of stock A increases while that of stock B remains the same, then we
can expect that the two will converge again. In this case, the optimal
decision to make is to sell stock A and buy stock B.

This paper deals with algorithms for statistical arbitrage. We present
novel techniques for this problem based on online learning.  In
section~\ref{sec:formulation}, we discuss a new formulation of the
problem of statistical arbitrage. We present online learning
algorithms for statistical arbitrage in Section~\ref{sec:algorithm}.

\section{Problem Formulation}
\label{sec:formulation}

Standard statistical arbitrage algorithms are based on statistic
deviations from the mean. However, for non-stationary distributions,
which are the typical stochastic processes we observe in the stock
market, estimating the mean or the standard deviation is a difficult
problem. The problem of learning with non-stationary distributions is
an active area of research \citep{KuznetsovMohri2015}. Thus, at the
heart of these algorithms, there is a key problem, that of
non-stationarity. One way to deal with such problems is to use the
notion of \emph{discrepancy} and to design sophisticated time series
algorithms such as those proposed by
\cite{KuznetsovMohri2015}. Alternatively, we can formulate
the problem as an instance of online learning which makes no
assumption about the distribution. Here, we adopt the latter
formulation since it admits the advantage of simplicity while
admitting strong regret-based learning guarantees.

\section{Algorithm}
\label{sec:algorithm}

\subsection{Online learning scenario}

The standard scenario of online learning can be described as follows:
at each round $t \in [1, T]$ the learner receives a point $x_t$ out of
an input space $X$ and makes a prediction $\hat y_t \in Y$ for the
label of $x_t$ out of the output space $Y$. He receives the true label
$y_t$ and incurs a loss $L(\hat y_t, y_t)$, where
$L\colon Y \times Y \to \Rset$ is the loss function associated with
the problem. No stochastic assumption is made about the input point
$x_t$ or its label $y_t$. The objective of the learner is to minimize
its \emph{regret}, that is the difference between its cumulative loss
over $T$ rounds and the loss of the best expert in hindsight.

We want to use online learning to determine when and how much to buy
and sell stocks when doing statistical arbitrage. We are going to use
the \emph{randomized weighted-majority algorithm}
\citep{LittlestoneWarmuth94} to help us make decisions in the stock
market because it admits a very favorable regret guarantee.

\subsection{Randomized weighted majority algorithm}

\subsubsection{Description}

In the randomized scenario of on-line learning, we assume that a set
$\sA = \set{1, \ldots, N}$ of $N$ actions or experts is available.  At
each round $t \in [T]$, an on-line algorithm $\Alg$ selects a
distribution $\p_t$ over the set of actions, receives a loss vector
$\mat{l}_t$, whose $i$th component $l_{t, i} \in [0, 1]$ is the loss
associated with action $i$, and incurs the expected loss
$L_t = \sum_{i = 1}^N p_{t, i} \, l_{t, i}$. The total loss incurred
by the algorithm over $T$ rounds is $\cL_T = \sum_{t = 1}^T L_t$. The
total loss associated to action $i$ is
$\cL_{T, i} = \sum_{t = 1}^T l_{t, i}$.  The minimal loss of a single
action is denoted by $\cL_T^{\min} = \min_{i \in \sA} \cL_{T, i}$. The
regret $R_T(\Alg)$ of the algorithm $]\Alg$ after $T$ rounds is then
typically defined by the difference of the loss of the algorithm and
that of the best single action:\footnote{Alternative definitions of
  the regret\index{regret} with comparison classes different from the
  set of single actions can be considered.}
\begin{equation*}
R_T = \cL_T - \cL_T^{\min}.
\end{equation*}
For this presentation of the algorithm, we consider specifically the
case of zero-one losses and assume that $l_{t, i} \in \set{0, 1}$ for
all $t \in [T]$ and $i \in \sA$.

The Randomized Weighted-Majority algorithm
\citep{LittlestoneWarmuth94} works as follows. The algorithm maintains
a distribution $p_{t}$ over a set of $N$ experts.  The original
distribution $p_1$ is initialized to be the uniform distribution.  At
each round, the loss assigned to each expert is revealed. The
algorithm incurs the expected loss over the experts, that is
$\sum_{i = 1}^N p_{t, i} \, l_{t, i}$, and then updates its
distribution on the set of experts by multiplying weight $p_{t, i}$ of
expert $i$ by $\beta \in (0, 1)$ when the expert committed a mistake
and leaving it unchanged otherwise. Next, the resulting probability
$p_{t + 1}$ is obtained after normalization. Figure~\ref{fig:RWM}
gives the pseudocode of the algorithm.

\begin{figure}[t]
\begin{ALGO}{Randomized-Weighted-Majority
\index{Randomized-Weighted-Majority algorithm}}{N} 
\DOFOR{i \EQ 1 \TO N}
\SET{w_{1, i}}{1}
\SET{p_{1, i}}{1/N}
\OD
\DOFOR{t \EQ 1 \TO T}
\CALL{Receive}{\bl_t}
\DOFOR{i \EQ 1 \TO N}
\IF{(l_{t, i} = 1)}
\SET{w_{t + 1, i}}{\beta w_{t, i}}
\ELSE
\SET{w_{t + 1, i}}{w_{t, i}}
\FI
\OD
\SET{W_{t + 1}}{\sum_{i = 1}^N w_{t + 1, i}}
\DOFOR{i \EQ 1 \TO N}
\SET{p_{t + 1, i}}{w_{t + 1, i}/W_{t + 1}}
\OD
\OD
\RETURN{\w_{T + 1}}
\end{ALGO}
\caption{Randomized weighted majority algorithm\index{Randomized-Weighted-Majority algorithm}\index{algorithm!randomized}.}
\label{fig:RWM}
\end{figure}

Its objective is to minimize its expected regret, that is the
difference between its cumulative loss and that of the best expert in
hindsight.

\subsubsection{Regret Guarantees}

The following theorem gives a strong guarantee on the
regret\index{regret} $R_T$ of the RWM
algorithm\index{Randomized-Weighted-Majority algorithm}, showing that
it is in $O(\sqrt{T \log N})$. The proof and the presentation
are based on \citep{MohriRostamizadehTalwalkar2012}.

\begin{theorem}
\label{th:randomized_weighted_majority}
Fix $\beta \in [1/2, 1)$. Then, for any $T \geq 1$, the loss of
algorithm RWM\index{Randomized-Weighted-Majority algorithm} on any sequence can be bounded as follows:
\begin{equation}
\cL_T \leq \frac{\log N}{1 - \beta} + (2 - \beta) \cL_T^{\min}.
\end{equation}
In particular, for $\beta = \max \set{1/2, 1 - \sqrt{(\log N)/T}}$, the loss can be bounded as:
\begin{equation}
\label{eq:RWM_regret}
\cL_T \leq \cL_T^{\min} + 2 \sqrt{T \log N}.
\end{equation}
\end{theorem}
\begin{proof}
  As in many proofs for deriving regret guarantees in on-line
  learning, we derive upper and lower bounds for the potential
  function\index{potential function} $W_t = \sum_{i = 1}^N w_{t, i}$,
  $t \in [T]$, and combine these bounds to obtain the result. By
  definition of the algorithm, for any $t \in [T]$, $W_{t + 1}$ can be
  expressed as follows in terms of $W_t$:
\begin{align*}
W_{t + 1} 
= \sum_{i \colon l_{t, i} = 0} w_{t, i} + \beta \sum_{i \colon
  l_{t.i} = 1} w_{t, i}
& = W_t + (\beta - 1) \sum_{i \colon l_{t, i} = 1} w_{t, i}\\
& = W_t + (\beta - 1) W_t \sum_{i \colon l_{t, i} = 1} p_{t, i}\\
& = W_t + (\beta - 1) W_t L_t\\
& = W_t (1 - (1 - \beta) L_t).
\end{align*}
Thus, since $W_1 = N$, it follows that $W_{T + 1} = N \prod_{t = 1}^T
(1 - (1 - \beta) L_t)$. On the other hand, the following lower bound
clearly holds: $W_{T + 1} \geq \max_{i \in [N]} w_{T + 1, i} =
\beta^{\cL_T^{\min}}$. This leads to the following inequality and
series of derivations after taking the $\log$ and using the
inequalities $\log (1 - x) \leq - x$ valid for all $x < 1$, and $-\log
(1 - x) \leq x + x^2$ valid for all $x \in [0, 1/2]$:
\begin{align*}
\beta^{\cL_T^{\min}} \leq N \prod_{t = 1}^T (1 - (1 - \beta) L_t)
\implies & \cL_T^{\min} \log \beta \leq \log N + \sum_{t = 1}^T \log
(1 - (1 - \beta) L_t) \\
\implies & \cL_T^{\min} \log \beta \leq \log N - (1 - \beta) \sum_{t =
  1}^T L_t \\
\implies & \cL_T^{\min} \log \beta \leq \log N - (1 - \beta) \cL_T\\
\implies & \cL_T  \leq \frac{\log N}{1 - \beta} - \frac{\log \beta}{1 -
\beta} \cL_T^{\min}\\
\implies & \cL_T  \leq \frac{\log N}{1 - \beta}  - \frac{\log (1 - (1 - \beta))}{1 -
\beta} \cL_T^{\min}\\
\implies & \cL_T  \leq \frac{\log N}{1 - \beta} + (2 - \beta) \cL_T^{\min}.
\end{align*}
This shows the first statement. Since $\cL_T^{\min} \leq T$, this also implies
\begin{equation}
\label{eq:L_T_bound}
\cL_T  \leq \frac{\log N}{1 - \beta} + (1 - \beta) T + \cL_T^{\min}.
\end{equation}
Differentiating the upper bound with respect to $\beta$ and setting it
to zero gives $\frac{\log N}{(1 - \beta)^2} - T = 0$, that is $\beta =
1 - \sqrt{(\log N)/T} < 1$. Thus, if $1 - \sqrt{(\log N)/T} \geq 1/2$,
$\beta_0 = 1 - \sqrt{(\log N)/T}$ is the minimizing value of $\beta$,
otherwise the boundary value $\beta_0 = 1/2$ is the optimal value. The
second statement follows by replacing $\beta$ with $\beta_0$ in
\eqref{eq:L_T_bound}. \qed
\end{proof}
The bound \eqref{eq:RWM_regret} assumes that the algorithm
additionally receives as a parameter the number of rounds $T$.  As we
shall see in the next section, however, there exists a general
\emph{doubling trick}\index{doubling trick} that can be used to relax this requirement at
the price of a small constant factor increase. Inequality~\ref{eq:RWM_regret}
can be written directly in terms of the regret\index{regret} $R_T$ of the RWM
algorithm\index{Randomized-Weighted-Majority algorithm}:
\begin{equation}
R_T \leq 2 \sqrt{T \log N}.
\end{equation}
Thus, for $N$ constant, the regret\index{regret} verifies
$R_T = O(\sqrt{T})$ and the \emph{average
  regret}\index{regret!average} or \emph{regret per
  round}\index{regret!per round} $R_T/T$ decreases as
$O(1/\sqrt{T})$. These results are known to be optimal.

\subsection{Online statistical arbitrage}

To use on-line learning algorithms to tackle the problem of statistic
arbitrage, we need to specify first the groups of correlated stocks on
which to do statistical arbitrage, since we intend to capitalize on
deviations the mean. The problem of determining such groups, pairs
in the simplest case, is a non-trivial task requiring itself a study
of the correlations between stocks over a large period of time. In the
following, we will assume that a group of such stocks is already
provided to us by specialists of the domain.

Second, we need to specify the set of experts used by our online
algorithm. At a high level, the regret analysis of the previous
section suggests choosing a relatively large set of experts ($N$)
since with the hope that the best expert among them would achieve a
small loss in hindsight. Otherwise, our benchmark would be poor and
our guarantee rather weak. While a very large $N$ affects the regret
guarantee of the RWM algorithm, that effect is rather mild is the
dependency on $N$ is only logarithmic.

How should the experts be chosen? One general way of selecting
experts is to let existing statistical arbitrage algorithms
serve as experts. The regret guarantees then ensure that for
large enough $T$, the loss of the online algorithm per round
would be very close to that of the best algorithm in hindsight,
since $\sqrt{T \log N}/T = \sqrt{\log N/T} \ll 1$ for large
$T$.

A widely used class of such algorithms act when the price of a stock
crosses a certain factor of the standard deviation from the mean
\citep{AvellanedaLee2008,Avellaneda2011}. For
example, if the price of a stock increases to a certain positive value
of $\gamma$, an expert would suggest selling, while the same expert
would suggest buying if the price went below the negative value of
$\gamma$. More specifically, these algorithms
are based on the \emph{$s$-score} for each stock which is defined
at time $t$ as follows for stock $i$ \citep{Avellaneda2011}:
\begin{equation}
s_i = \frac{X_i(t) - \mu}{\sigma_{i}},
\end{equation}
where $X_i(t)$ denotes the value of the stock at time $t$, $\mu$
the mean value, and $\sigma_{i}$ the standard deviation. These
algorithms then proceed according to the following rules \citep{Avellaneda2011}:
\begin{equation}
\begin{cases}
\text{Open long position} & \text{if } s_i <-1.25\\
\text{Open short position} & \text{if } s_i >+1.25\\
\text{Close long position} & \text{if } s_i > -0.50\\
\text{Close short position} & \text{if } s_i < +0.50.
\end{cases}
\end{equation}
Thus, these rules determine when to act in statistical arbitrage for
different groups of stocks as well as in different time periods.

\begin{figure}[t]
  \centering
  \includegraphics[scale=0.25]{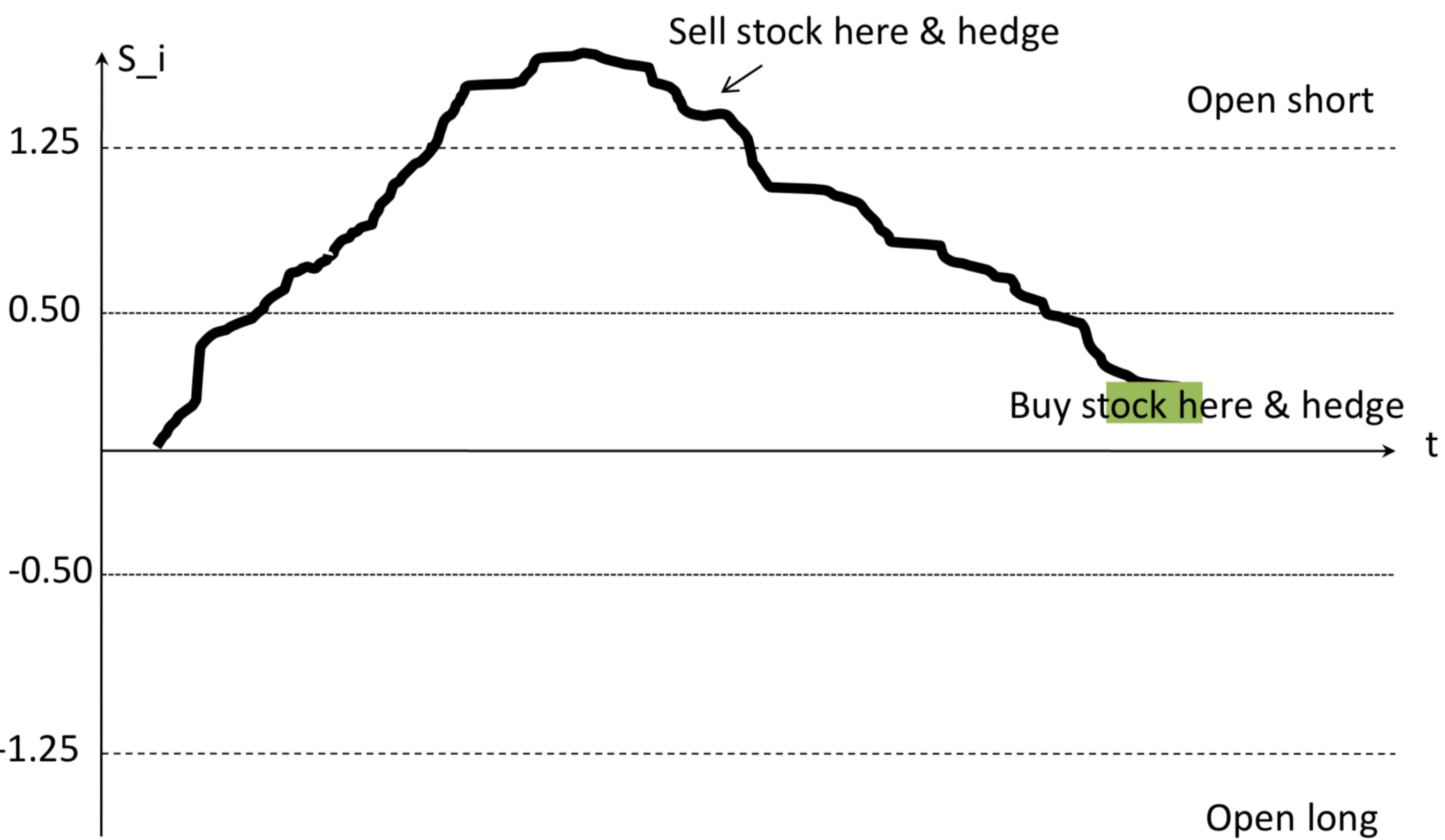}
  \caption{Illustration of some standard statistical arbitrage algorithms.}
\label{fig:xp}
\end{figure}

Instead of the specific values $.5$ and $1.25$ in these rules, we will
consider a family of algorithms parameterized by two positive real
values $\gamma_1$ and $\gamma_2$ with $\gamma_1 < \gamma_2$. For
$\gamma_1 = .5$ and $\gamma_2 = 1.25$, the algorithm coincides
with the one described above. As suggested by the discussion about
the size $N$ of the experts pool, we suggest considering a relatively 
large set $N$ of pairs $(\gamma_1, \gamma_2)$ and therefore experts.

\section{Conclusion}

We presented a general solution for statistical arbitrage based online
learning algorithms and concepts. We gave a full algorithmic and
theoretical analysis of the problem within this framework.  A key
advantage of our algorithm is that it does not require strong
stationary assumptions about the stochastic process, which often do
not hold in practice.  We hope to present preliminary experimental
results in support of our algorithms in the future and plan to discuss
a number of enhancements, including leveraging online algorithms
competing against sequence of experts, which allow for non-static
experts to serve as benchmarks in the definition of the regret
\citep{HerbsterWarmuth1998,HazanSeshadhri2009,AdamskiyKoolenChernovVovk2012\ignore{,MohriYang2018}}.

\section*{Acknowledgments}

I warmly thank Harr Chen and his team at Vatic Labs who gave me an
excellent general introduction to computational finance and in
particular familiarized me with the interesting problem of statistical
arbitrage.
\cite{*}
\newpage

\bibliographystyle{icml2014} 
\bibliography{arbb}

\begin{thebibliography}{17}
\providecommand{\natexlab}[1]{#1}
\providecommand{\url}[1]{\texttt{#1}}
\expandafter\ifx\csname urlstyle\endcsname\relax
  \providecommand{\doi}[1]{doi: #1}\else
  \providecommand{\doi}{doi: \begingroup \urlstyle{rm}\Url}\fi

\bibitem[Adamskiy et~al.(2012)Adamskiy, Koolen, Chernov, and
  Vovk]{AdamskiyKoolenChernovVovk2012}
Adamskiy, Dmitry, Koolen, Wouter~M, Chernov, Alexey, and Vovk, Vladimir.
\newblock A closer look at adaptive regret.
\newblock In \emph{ALT}, pp.\  290--304, 2012.

\bibitem[Avellaneda(2011)]{Avellaneda2011}
Avellaneda, Marco.
\newblock Risk and portfolio management: Statisical arbitrage, 2011.
\newblock URL
  \url{https://www.math.nyu.edu/faculty/avellane/Lecture8Risk2011.pdf}.

\bibitem[Avellaneda \& Lee(2008)Avellaneda and Lee]{AvellanedaLee2008}
Avellaneda, Marco and Lee, Jeong-Hyun.
\newblock Statistical arbitrage in the u.s. equities market.
\newblock Technical report, Courant Institute of Mathematical Sciences, 2008.

\bibitem[Beck \& Tetruashvili(2013)Beck and Tetruashvili]{BeckTetruashvili2013}
Beck, Amir and Tetruashvili, Luba.
\newblock On the convergence of block coordinate descent type methods.
\newblock \emph{SIAM Journal on Optimization}, 23\penalty0 (4):\penalty0
  2037--2060, 2013.

\bibitem[Damghani \& Kos(2013)Damghani and Kos]{DamghaniKos2013}
Damghani, Babak~Mahdavi and Kos, Andrew.
\newblock De-arbitraging with a weak smile: Application to skew risks.
\newblock \emph{Wilmott Magazine}, pp.\  40--49, 2013.

\bibitem[Fernholz \& Cary~Maguire(2007)Fernholz and
  Cary~Maguire]{FernholzMaguire2007}
Fernholz, Robert and Cary~Maguire, Jr.
\newblock The statistics of statistical arbitrage.
\newblock \emph{Financial Analysts Journal}, pp.\  46--52, 2007.

\bibitem[Hazan \& Seshadhri(2009)Hazan and Seshadhri]{HazanSeshadhri2009}
Hazan, Elad and Seshadhri, Comandur.
\newblock Efficient learning algorithms for changing environments.
\newblock In \emph{Proceedings of ICML}, pp.\  393--400. ACM, 2009.

\bibitem[Herbster \& Warmuth(1998)Herbster and Warmuth]{HerbsterWarmuth1998}
Herbster, Mark and Warmuth, Manfred~K.
\newblock Tracking the best expert.
\newblock \emph{Machine Learning}, 32\penalty0 (2):\penalty0 151--178, 1998.

\bibitem[Kuznetsov \& Mohri(2015)Kuznetsov and Mohri]{KuznetsovMohri2015}
Kuznetsov, Vitaly and Mohri, Mehryar.
\newblock Learning theory and algorithms for forecasting non-stationary time
  series.
\newblock In \emph{NIPS}, 2015.

\bibitem[Littlestone \& Warmuth(1994)Littlestone and
  Warmuth]{LittlestoneWarmuth94}
Littlestone, Nick and Warmuth, Manfred~K.
\newblock The weighted majority algorithm.
\newblock \emph{Information and Computation}, 108\penalty0 (2):\penalty0
  212--261, 1994.

\bibitem[Lo(2010)]{Lo2010}
Lo, Andrew~W.
\newblock \emph{Hedge Funds: An Analytic Perspective.}
\newblock Princeton University Press, 2010.

\bibitem[Luo \& Tseng(1992)Luo and Tseng]{LuoTseng1992}
Luo, Zhi-Quan and Tseng, Paul.
\newblock On the convergence of the coordinate descent method for convex
  differentiable minimization.
\newblock \emph{Journal of Optimization Theory and Applications}, 72\penalty0
  (1):\penalty0 7--35, 1992.

\bibitem[Mohri et~al.(2012)Mohri, Rostamizadeh, and
  Talwalkar]{MohriRostamizadehTalwalkar2012}
Mohri, Mehryar, Rostamizadeh, Afshin, and Talwalkar, Ameet.
\newblock \emph{Foundations of Machine Learning}.
\newblock Adaptive computation and machine learning. {MIT} Press, 2012.

\bibitem[Nesterov(2012)]{Nesterov2012}
Nesterov, Yurii.
\newblock Efficiency of coordinate descent methods on huge-scale optimization
  problems.
\newblock \emph{SIAM Journal on Optimization}, 22\penalty0 (2):\penalty0
  341--362, 2012.

\bibitem[R{\"a}tsch et~al.(2001)R{\"a}tsch, Mika, and
  Warmuth]{RaetschMikaWarmuth2001}
R{\"a}tsch, Gunnar, Mika, Sebastian, and Warmuth, Manfred~K.
\newblock On the convergence of leveraging.
\newblock In \emph{NIPS}, pp.\  487--494, 2001.

\bibitem[Tseng(2001)]{Tseng2001}
Tseng, P.
\newblock Convergence of a block coordinate descent method for
  nondifferentiable minimization.
\newblock \emph{Journal Optimization Theory and Applications}, pp.\  475--494,
  2001.

\bibitem[Wu(2007)]{Wu2007}
Wu, Liuren.
\newblock Statistical arbitrage based on no-arbitrage models, 2007.
\newblock URL \url{http://faculty.baruch.cuny.edu/lwu/papers/StatArb.pdf}.

\end{thebibliography}
\end{document}